\documentclass[9.5bp,journal,compsoc]{IEEEtran}
\usepackage{epsfig}
\usepackage{graphicx}
\usepackage[pagebackref=false,breaklinks=true,colorlinks,linkcolor={black},citecolor={black},urlcolor={blue},bookmarks=false]{hyperref}
\usepackage{cite}
\usepackage{microtype}
\usepackage{subfigure}
\usepackage{url}

\usepackage{array}
\usepackage[normalem]{ulem}
\usepackage{ulem}
\usepackage{amsmath}
\usepackage{graphicx,subfigure}
\usepackage{longtable}
\usepackage{rotating}
\usepackage{multirow}
\usepackage{makecell}
\usepackage{bm}
\usepackage{float}
\usepackage{flushend}
\usepackage{xcolor}
\usepackage{soul}
\usepackage{booktabs}
\usepackage{threeparttable}
\usepackage{color}
\usepackage{algorithm}
\usepackage{algorithmic}
\usepackage{tabularx}
\usepackage{amsfonts}
\usepackage{epstopdf}
\usepackage{latexsym}
\usepackage{amssymb}
\usepackage{footnote}

\usepackage{pifont}

\newtheorem{theorem}{Theorem}
\newtheorem{lemma}{Lemma}
\newtheorem{proof}{Proof}
\newtheorem{definition}{Definition}

\renewcommand{\IEEEQED}{\IEEEQEDopen}

\usepackage{threeparttable}

\def\Sp{{\scriptsize{\textcircled{{\emph{\tiny{\textbf{Sp}}}}}}}}
\def\etal{\textit{et al.}}
\def\ie{\textit{i.e.}}

\begin{document}

\title{Multi-View Clustering via Semi-non-negative Tensor Factorization}

\author{Jing Li,
        Quanxue~Gao,
        Qianqian~Wang,
        Wei~Xia,
        and~Xinbo~Gao,~\IEEEmembership{Senior Member,~IEEE}
\IEEEcompsocitemizethanks{
\IEEEcompsocthanksitem This work was supported in part by National Natural Science Foundation of China under Grants 62176203, 62036007; in part by the Open Project Program of the National Laboratory of Pattern Recognition (NLPR) under Grant 202200035; in part by Natural Science Basic Research Plan in Shaanxi Province (Grant 2020JZ-19), and in part by the Fundamental Research Funds for the Central Universities and the Innovation Fund of Xidian University. (Corresponding author: Q. Gao.)\protect

\IEEEcompsocthanksitem J. Li, Q. Gao, Q. Wang, and W. Xia are with the State Key laboratory of Integrated Services Networks, Xidian University, Xi'an 710071, China (e-mail: xd.weixia@gmail.com; qxgao@xidian.edu.cn).\protect

\IEEEcompsocthanksitem X. Gao is with the Chongqing Key Laboratory of Image Cognition, Chongqing University of Posts and Telecommunications, Chongqing 400065, China (e-mail: gaoxb@cqupt.edu.cn), and with the School of Electronic Engineering, Xidian University, Xi’an 710071, China (e-mail: xbgao@mail.xidian.edu.cn).\protect}%
\thanks{Manuscript received XXXX; revised XXXX; accepted XXXX.}}

\markboth{IEEE TRANSACTIONS}%
{Shell \MakeLowercase{\textit{Xia et al.}}: Multi-View Clustering via Semi-non-negative Tensor Factorization}

\IEEEtitleabstractindextext{%
\begin{abstract}
Multi-view clustering (MVC) based on non-negative matrix factorization (NMF) and its variants have received a huge amount of attention in recent years due to their advantages in clustering interpretability. However, existing NMF-based multi-view clustering methods perform NMF on each view data respectively and ignore the impact of between-view. Thus, they can't well exploit the within-view spatial structure and between-view complementary information. To resolve this issue, we present semi-non-negative tensor factorization (Semi-NTF) and develop a novel multi-view clustering based on Semi-NTF with one-side orthogonal constraint. Our model directly performs Semi-NTF on the 3rd-order tensor which is composed of anchor graphs of views. Thus, our model directly considers the between-view relationship. Moreover, we use the tensor Schatten $p$-norm regularization as a rank approximation of the 3rd-order tensor which characterizes the cluster structure of multi-view data and exploits the between-view complementary information. In addition, we provide an optimization algorithm for the proposed method and prove mathematically that the algorithm always converges to the stationary KKT point. Extensive experiments on various benchmark datasets indicate that our proposed method is able to achieve satisfactory clustering performance.
\end{abstract}

}

\maketitle
\IEEEdisplaynontitleabstractindextext
\IEEEpeerreviewmaketitle

\IEEEraisesectionheading{\section{Introduction}\label{sec:introduction}}
\IEEEPARstart{A}{s} one of the most typical methods in unsupervised learning, clustering has a wide scope of application~\cite{oyelade2010application,chang2014multi,alashwal2019application} to assign data to different clusters according to the information describing the objects and their relationships.
Non-negative matrix factorization (NMF)~\cite{lee1999learning} is one of the representative methods of clustering, which is proved to be equivalent to K-means clustering~\cite{ding2005equivalence}.
Despite the widespread use of NMF, there are some drawbacks that have prompted some variants of NMF~\cite{ding2006orthogonal,cichocki2006multilayer,pauca2006nonnegative,ding2008convex,cai2010graph}.

In particular, Ding~\etal presented the one-side G-orthogonal NMF~\cite{ding2006orthogonal}. It can guarantee the uniqueness of the solution of matrix factorization and has excellent clustering interpretation. Also, Ding~\etal proposed the semi-NMF~\cite{ding2008convex}. The data matrix and one of the factor matrices are unconstrained, which allows semi-NMF to be more suitable for applications where the input data is mixed with positive and negative numbers.
Although the above methods can achieve outstanding clustering performance, they are all single-view clustering methods and cannot be adopted straightforwardly for multi-view clustering.

Multi-view clustering tends to achieve superior performance compared to traditional single-view clustering owing to the capability to leverage the complementary information embedded in the different views.
Considering the superiority of MVC and NMF, lots of NMF-based multi-view clustering methods have been proposed~\cite{greene2009matrix,wang2011fast,liu2013multi,wang2015feature,he2016learning,wang2016multi,han2017bilateral,yang2021fast}.
The NMF-based multi-view clustering methods can save time and space because it is unnecessary to construct affinity graphs while graph-based methods have to. However, usually, they decompose the original data matrix directly, which leads to a dramatic reduction in the efficiency of the algorithm when the dimension of the original data is huge.

To this end, Yang~\etal presented a fast multi-view clustering method based on NMF and an anchor graph called FMCNOF~\cite{yang2021fast}. It performs NMF on the obtained anchor graph. Due to the fact that the dimension of the anchor graph is considerably smaller than the original affinity graph, it follows that the clustering efficiency can be improved.
Despite the excellent performance of FMCNOF, it suffers from several shortcomings. 

To fix the aforesaid issues, we proposed a semi-non-negative tensor factorization model with one-side orthogonal constraint. Specifically, we abandon the usage of K-means to obtain anchors and adopt a novel anchor selection strategy to obtain fixed anchors. As is well-known, there exist two ways of NMF-based multi-view clustering methods. One is to integrate different views first and then implement the NMF on the integrated matrix; the other is to perform the NMF on different views separately and then integrate the result from each view. Both ways are essentially applications of NMF on a single view, and both need to reduce the multi-view data into two-dimension matrices in the process, which causes the loss of the original spatial information.
In order to fully utilize the spatial information of multi-view data, we extend non-negative matrix factorization to the 3rd-order tensor to make it more suitable for processing multi-view data. In addition, inspired by the performance of tensor Schatten $p$-norm~\cite{gao2020enhanced}, we adopt the 3rd-order tensor Schatten $p$-norm regularization assembled by the clustering indicator matrix to sufficiently explore the complementary information embedded in the multi-view data of different views.
A large number of experiments have shown that our methods have excellent clustering performance.

The main contributions are summarized below:
\begin{itemize}
  \item We introduce semi-non-negative tensor factorization, which considers the between-view relationship directly. Also, we use tensor Schatten $p$-norm regularization to characterize the cluster structure of multi-view data and can exploit the complementary information of between views.
  \item We regard the anchor graph obtained from the original data as the input of the non-negative matrix factorization, which reduces the complexity of our proposed algorithm considerably.
  \item We provide an optimization algorithm for the proposed method and prove it always converges to the KKT stationary point mathematically. The effectiveness of its application on tensorial  G-orthogonal non-negative matrix factorization is demonstrated by extensive experiments.
\end{itemize}

\section{Related work}\label{Related work}
In recent years, multi-view clustering (MVC) has received increasing attention due to its excellent clustering performance. Also, non-negative matrix factorization (NMF) is an efficient technique in single-view clustering, which can generate excellent clustering results that are easy to interpret, and many NMF-based variants have been proposed. Therefore, multi-view clustering-based NMF and its variants have attracted tremendous interest recently.


As the first investigation of the multi-view clustering method based on joint NMF, multiNMF~\cite{liu2013multi} implements NMF at each view and pushes the different clustering results of each view to a consensus. It provides a new viewpoint for the subsequent NMF-based MVC methods.
Influenced by multiNMF, He~\etal proposed a multi-view clustering method combining NMF with similarity~\cite{he2016learning}. It implements NMF on each view as in multiNMF. In addition, it sets a weight for a different view and introduces a similarity matrix of data points to extract consistent information from different views.
To better detect the geometric structure of inner-view space, Wang~\etal~\cite{wang2015feature} introduced graph regularization into the NMF-based multi-view clustering method to improve clustering performance.
Considering the above work, Wang~\etal~\cite{wang2016multi} proposed a graph regularization multi-view clustering method based on concept factorization (CF).
CF is a variant of NMF and it is suitable for handling data containing negative.

As the size of data grows, lots of methods to accelerate matrix factorization are presented. Wang~\etal~\cite{wang2011fast} proposed a fast non-negative matrix triple factorization method. It constrains the factor matrix of NMF to a clustering indicator matrix, thereby avoiding the post-processing of the factor matrix.
Inspired by the work of Wang, Han~\etal~\cite{han2017bilateral} constrained the intermediate factor matrix in the triple factorization to a diagonal matrix, reducing the number of matrix multiplications in the solution process.
Considering that previous NMF-based multi-view clustering methods are performed directly on the original data, Yang~\etal~\cite{yang2021fast} introduced an anchor graph as the input of G-orthogonal NMF. Since the number of anchors is much smaller than that of the original data, the matrix factorization of the anchor graph can indeed be more efficient.

Despite the fact that existing NMF-based multi-view clustering methods can perform the clustering tasks excellently, it is still essential work to take full advantage of the complementary information and spatial structure among different views.

\section{Notations}\label{Notations}
We introduce the notations used throughout this paper. We use bold calligraphy letters for 3rd-order tensors, ${\bm{\mathcal {H}}} \in{\mathbb{R}} {^{{n_1} \times {n_2} \times {n_3}}}$, bold upper case letters for matrices, ${\mathbf{H}}$, bold lower case letters for vectors, ${\bf{h}}$, and lower case letters such as ${h_{ijk}}$ for the entries of ${\bm{\mathcal {H}}}$. Moreover, the $i$-th frontal slice of ${\bm{\mathcal {H}}}$ is ${\bm{\mathcal {H}}}^{(i)}$. $\overline {{\bm{\mathcal {H}}}}$ is the discrete Fourier transform (DFT) of ${\bm{\mathcal {H}}}$ along the third dimension, $\overline {{\bm{\mathcal {H}}}} = \mathrm{fft}({{\bm{\mathcal H}}},[\ ],3)$. Thus, $\bm{{\mathcal H}} = \mathrm{ifft}({\overline {\bm{\mathcal H}}},[\ ],3)$. The trace and transpose of matrix $\mathbf{H}$ are expressed as $\mathrm{tr}(\mathbf{H})$ and  $\mathbf{H}^{\mathrm{T}}$. The F-norm of ${\bm{\mathcal H}}$ is denoted by ${\left\| {\bm{\mathcal H}}\right\|_F}$.

\begin{definition}[t-product \cite{kilmer2011factorization}]\label{def:t-prod}
Suppose ${\bm{\mathcal{A}}}\in\mathbb{R}^{n_1\times m\times n_3}$ and ${\bm{\mathcal{B}}}\in \mathbb{R}^{m\times n_2\times n_3}$, the t-product ${\bm{\mathcal{A}}}*{\bm{\mathcal{B}}}\in\mathbb{R}^{n_1\times n_2\times n_3}$ is given by
\begin{align*}
{\bm{\mathcal{A}}}*{\bm{\mathcal{B}}} = \mathrm{ifft}(\mathrm{bdiag}(\overline{\mathbf A}\overline{\mathbf B}),[\ ],3),
\end{align*}
where $\overline{\mathbf A}=\mathrm{bdiag}(\bm{\overline{\mathcal{A}}})$ and it denotes the block diagonal matrix. The blocks of $\overline{\mathbf A}$ are frontal slices of $\bm{\overline{\mathcal{A}}}$.
\end{definition}

\begin{definition}\label{tensorSpNorm}~\cite{gao2020enhanced}
Given ${\bm{\mathcal H}}\in{\mathbb{R}}^{n_1 \times n_2 \times n_3}$, $h=\min(n_1,n_2)$, the tensor Schatten $p$-norm of  ${\bm{\mathcal H}}$ is defined as
\begin{small}
\begin{equation}
\begin{array}{c}
{\left\| {\bm{\mathcal H}} \right\|_{{\Sp}}} = {\left( {\sum\limits_{i = 1}^{{n_3}} {\left\| {{{\bm{\overline {\cal H} }^{(i)}}}} \right\|_{{\Sp}}^p} } \right)^{\frac{1}{p}}} = {\left( {\sum\limits_{i = 1}^{{n_3}} {\sum\limits_{j = 1}^h {{\sigma _j}{{\left( {{\bm{\overline {\cal H} }^{(i)}}} \right)}^p}} } } \right)^{\frac{1}{p}}},
\end{array}
\label{4}
\end{equation}
\end{small}
where, $0 < p \leqslant 1$, ${\sigma _j}(\overline{\bm{\mathcal H}}^{(i)})$ denotes the j-th singular value of $\overline{\bm{\mathcal H}}^{(i)}$.
\label{definition1}
\end{definition}
It should be pointed out that  for $0 < p \leqslant 1$ when $p$ is appropriately chosen, the  Schatten $p$-norm provides quite effective improvements for a tighter approximation of the rank function~\cite{zha2020benchmark,xie2016weighted}.

\section{Methodology}\label{Methodology}

\subsection{Motivation and Objective}
Non-negative matrix factorization (NMF) was initially presented as a dimensionality reduction method, and it is commonly employed as an efficient latent feature learning technique recently.
Generally speaking, given a non-negative matrix ${\bf{X}}$, the target of NMF is to decompose ${\bf{X}}$ into two non-negative matrices,
\begin{equation}
    {\mathbf{X}} \approx {\mathbf{H}}{\mathbf{G}}^{\mathrm{T}}
\end{equation}
where ${\mathbf{X}} \in {\mathbb{R}}_+^{{n} \times {p}}$, ${\mathbf{H}} \in {\mathbb{R}}_+^{{n} \times {k}}$ and ${\mathbf{G}} \in {\mathbb{R}}_+^{{p} \times {k}}$. ${\mathbb{R}}_+^{{n} \times {p}}$ means $n$-by-$p$ matrices with elements are all nonnegative.

In order to approximate the matrix before and after factorization, $\ell_2$-norm and F-norm are frequently adopted as the objective function for the NMF. Considering that F-norm can make the model optimization easier, we use F-norm to construct the objective function.
\begin{equation}
    \min_{\mathbf{H} \geqslant 0, \mathbf{G} \geqslant 0}{\left\| {\mathbf X} - {\mathbf{H}} {\mathbf{G}}^{\mathrm{T}} \right \|}_F^2
\end{equation}

With the extensive use of NMF, more and more variants of NMF have emerged, among which are G-orthogonal NMF~\cite{ding2006orthogonal} and Semi-NMF~\cite{ding2008convex}. By imposing an orthogonality constraint on one of the factor matrices in NMF, we obtain the objective function of the one-side orthogonal NMF,
\begin{equation}
    \min_{\mathbf{H} \geqslant 0, \mathbf{G} \geqslant 0}{\left\| {\mathbf X} - {\mathbf{H}} {\mathbf{G}}^{\mathrm{T}} \right \|}_F^2, \quad\textrm{s.t.}  \quad \mathbf{H}^{\mathrm{T}} \mathbf{H}=\mathbf{I}.
\end{equation}
If we relax the nonnegative constraint on one of the factor matrices in the NMF and the input matrix $\mathbf X$ can also be mixed positive and negative, then we can get Semi-NMF. Semi-NMF can be adapted to process input data that has mixed symbols. For G-orthogonal NMF and Semi-NMF, Ding~\etal~\cite{ding2006orthogonal} presented the following lemma:
\begin{lemma}
G-orthogonal NMF and Semi-NMF are all relaxation of K-means clustering, and the main advantages of G-orthogonal NMF are (1) Uniqueness of the solution; (2) Excellent clustering interpretability.
\end{lemma}

Taking into account the one-side orthogonal NMF, we relax the nonnegative constraints on $\mathbf X$ and $\mathbf G$.
Moreover, inspired by FMCNOF~\cite{yang2021fast}, we construct the anchor graph $\mathbf S$ obtained from the original data $\bf X$ as the input of matrix factorization. Compared with the original data, the number of anchors is much smaller, therefore, by adopting the anchor graph constructed by anchors and original data points as the input of matrix factorization, we can reduce the computational complexity of the algorithm effectively.
\begin{equation}
    \min_{\mathbf{H} \geqslant 0}{\left\| {\mathbf S} - {\mathbf{H}} {\mathbf{G}}^{\mathrm{T}} \right \|}_F^2, \quad\textrm{s.t.}  \quad \mathbf{H}^{\mathrm{T}} \mathbf{H}=\bf{I},
\end{equation}
where ${\mathbf S} \in {\mathbb{R}}^{{n} \times {m}}$, ${\mathbf{H}} \in {\mathbb{R}}^{{n} \times {k}}$ and ${\mathbf{G}} \in {\mathbb{R}}^{{m} \times {k}}$, $m$ is the number of anchors and we consider ${\mathbf{H}}$ as the cluster indicator matrix for clustering rows as described in~\cite{ding2006orthogonal}. We will introduce the details of anchor selection and the construction of the anchor graph in the appendix.

As described in the previous section, the existing NMF-based multi-view clustering methods are essentially a matrix factorization on a single view combined with the integration of multiple views. It causes the loss of the original spatial structure of the multi-view data. We extend NMF to the 3rd-order tensor, which can process the multi-view data directly and can also take full advantage of the original spatial structure of the multi-view data. The objective function of tensorial  one-side orthogonal non-negative matrix factorization is written in the following form:
\begin{equation}\label{tensorNMF}
    \min_{\bm{\mathcal H}\geqslant 0} {\left\| \bm{\mathcal S} - \bm{\mathcal H}\bm{\mathcal G}^{\mathrm{T}} \right \|}_F^2, \quad \textrm{s.t.} \quad \bm{\mathcal H}^{\mathrm{T}}\bm{\mathcal H} = \bm{\mathcal I},
\end{equation}
The 3rd-order tensor construction process is illustrated in Fig \ref{fig_tensor}.

\begin{figure}[t]
	\centering
	\includegraphics[width=0.75\linewidth]{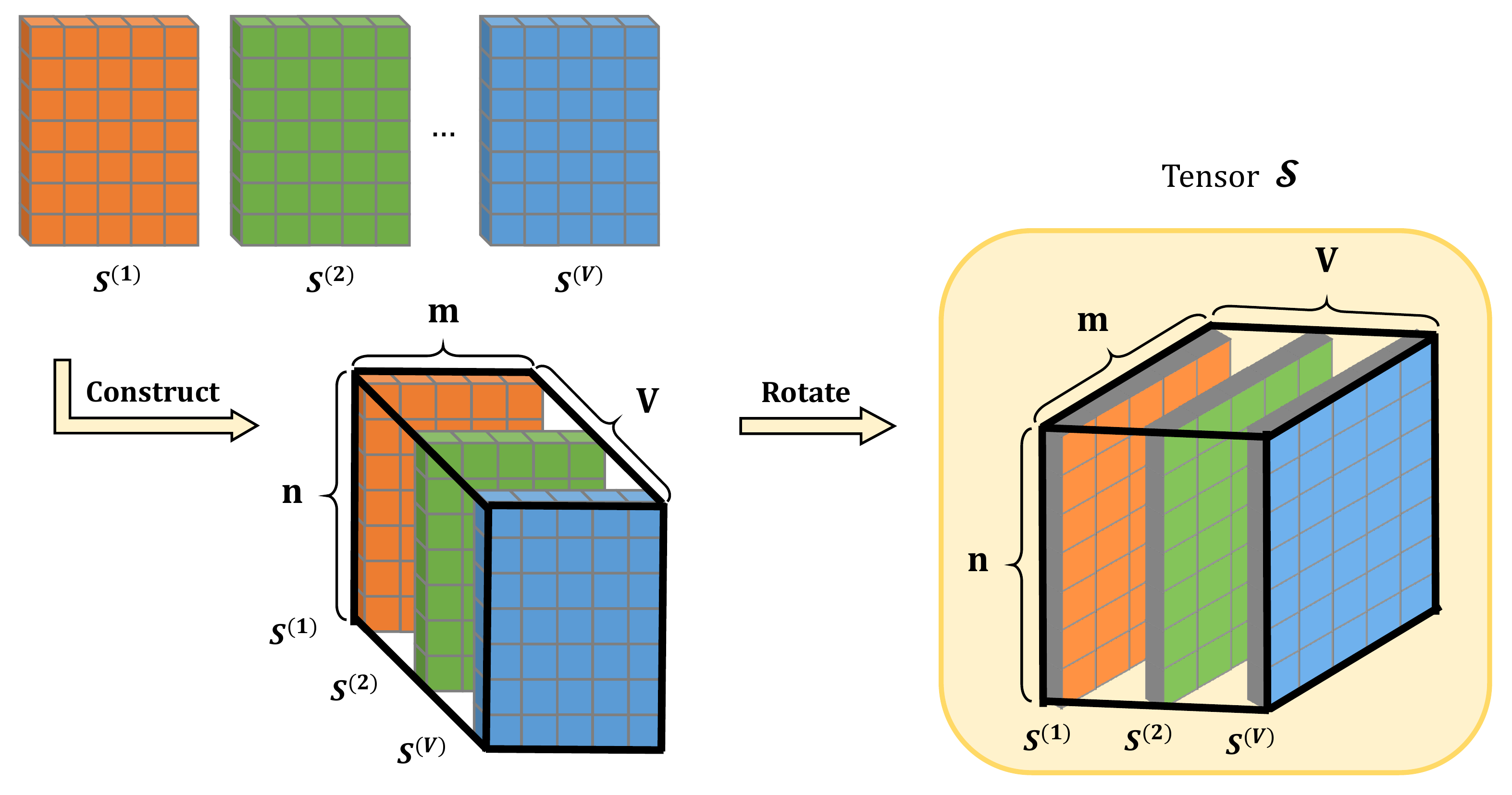}
	\caption{3rd-order tensor construction process.}
	\label{fig_tensor}
	\vspace{-5mm}
\end{figure}

After obtaining the cluster indicator matrices of different views ${\mathbf H^{(v)}}$, which is concatenated into a 3rd-order tensor $\bm{\mathcal H}$, we have to employ a strategy to push them into a consistent.
Especially motivated by the outstanding performance of the tensor Schatten $p$-norm~\cite{gao2020enhanced}, we minimize the tensor Schatten $p$-norm of $\bm{\mathcal H}$ aiming to diminish the divergence of the cluster indicator matrices from different views:
\begin{equation}\label{tensorIndicator}
    \min_{\bm{\mathcal H}\geqslant 0} {\left \|{\bm{\mathcal H}}\right \|}_\Sp^p,
\end{equation}
where $0<p \leqslant 1$ and the definition of the tensor Schatten $p$-norm is illustrated in Definition \ref{tensorSpNorm}.

The tensor Schatten $p$-norm can take full advantage of the complementary information embedded in different views, as well as the spatial structure of the multi-view data owing to the straightforward processing on the tensorial  form of the cluster indicator matrix.
By combining  (\ref{tensorNMF}) with  (\ref{tensorIndicator}), we reach the final objective function of the tensorial  one-side orthogonal NMF:
\begin{equation}
    \begin{aligned}
        &\min {\left\| \bm{\mathcal S} - \bm{\mathcal H}\bm{\mathcal G}^{\mathrm{T}} \right \|}_F^2 + \lambda {\left \|{\bm{\mathcal H}}\right \|}_\Sp^p  \\
        &\quad \emph{\textrm{s.t.}} \quad  \bm{\mathcal H}\geqslant 0, \bm{\mathcal H}^{\mathrm{T}}\bm{\mathcal H} = \bm{\mathcal I}
    \end{aligned}
\end{equation}
where $\lambda$ is the hyperparameter of the Schatten $p$-norm term.

\subsection{Optimization}
Inspired by Augmented Lagrange Multiplier (ALM), we introduce two auxiliary variables $\bm{\mathcal Q}$ and $\bm{\mathcal J}$ and let $\bm{\mathcal H} = \bm{\mathcal Q}$, $\bm{\mathcal H} = \bm{\mathcal J}$, respectively, where $\bm{\mathcal Q}^{\mathrm{T}} \bm{\mathcal Q} = \bm{\mathcal I}$. Then, we rewrite the model as the  following unconstrained problem:
\begin{equation}\label{objective function}
	\begin{aligned}
		&\min \bm{\mathcal L} (\bm{\mathcal Q}, \bm{\mathcal H}, \bm{\mathcal G}, \bm{\mathcal J}) \\
		&=\min_{\bm{\mathcal H}\geqslant 0, \bm{\mathcal Q}^{\mathrm{T}}\bm{\mathcal Q} = \bm{\mathcal I}} {\left\| \bm{\mathcal S} - \bm{\mathcal Q}\bm{\mathcal G}^{\mathrm{T}} \right \|}_F^2 + \lambda{\| \bm{\mathcal J}\|}_\Sp^p  \\
		& \quad + \frac{\mu}{2} {\left\| \bm{\mathcal H}-\bm{\mathcal Q} + \frac{\bm{\mathcal Y_1}}{\mu}\right \|}_F^2 + \frac{\rho}{2} {\left\| \bm{\mathcal H}-\bm{\mathcal J} + \frac{\bm{\mathcal Y_2}}{\rho}\right \|}_F^2,
	\end{aligned}
\end{equation}
where $\bm{\mathcal Y_1}$, $\bm{\mathcal Y_2}$ represent Lagrange multipliers and $\mu$, $\rho$ are the penalty parameters. The optimization process can therefore be separated into four steps:

$\bullet$\textbf{Solve $\bm{\mathcal G}$ with fixed $\bm{\mathcal Q}, \bm{\mathcal H}, \bm{\mathcal J}$.}  (\ref{objective function}) becomes:
\begin{equation}\label{solveG1}
	\begin{aligned}
		\min {\left\| \bm{\mathcal S} - \bm{\mathcal Q}\bm{\mathcal G}^{\mathrm{T}} \right \|}_F^2
	\end{aligned}
\end{equation}

After being implemented with discrete Fourier transform (DFT) along the third dimension.  the equivalent representation of  (\ref{solveG1}) in the frequency domain becomes:
\begin{equation}\label{solveG2}
	\begin{aligned}
		\min \sum_{v=1}^{V} {\left\| \bm{\mathcal {\overline{S}}}^{(v)} - \bm{\mathcal {\overline{Q}}}^{(v)}(\bm{\mathcal {\overline{G}}}^{(v)})^{\mathrm{T}}  \right \|}_F^2,
	\end{aligned}
\end{equation}
where $\bm{\mathcal {\overline{G}}} = \mathrm{fft}({{\bm{\mathcal G}}},[ \ ],3)$, and the others in the same way.

Let $\Phi= {\left\| \bm{\mathcal {\overline{S}}}^{(v)} - \bm{\mathcal {\overline{Q}}}^{(v)}(\bm{\mathcal {\overline{G}}}^{(v)})^{\mathrm{T}}  \right \|}_F^2$, we can obviously get the following equation:
\begin{equation}\label{solveG3}
    \begin{aligned}
        \Phi &= \mathrm{tr}\left((\bm{\mathcal {\overline{S}}}^{(v)})^{\mathrm{T}} \bm{\mathcal {\overline{S}}}^{(v)} \right ) - 2\mathrm{tr} \left((\bm{\mathcal {\overline{Q}}}^{(v)})^{\mathrm{T}} \bm{\mathcal {\overline{S}}}^{(v)} \bm{\mathcal {\overline{G}}}^{(v)} \right )  \\
        &+ \mathrm{tr} \left( (\bm{\mathcal {\overline{Q}}}^{(v)})^{\mathrm{T}} \bm{\mathcal {\overline{Q}}}^{(v)} \right).
    \end{aligned}
\end{equation}
Setting the derivative $\partial \Phi / \partial \bm{\mathcal {\overline{Q}}}^{(v)} = 0$ gives $ 2 \bm{\mathcal {\overline{G}}}^{(v)} - 2(\bm{\mathcal {\overline{S}}}^{(v)})^{\mathrm{T}} \bm{\mathcal {\overline{Q}}}^{(v)} = 0$. So the solution of  (\ref{solveG2}) is:
\begin{equation}\label{solveG}
    \bm{\mathcal {\overline{G}}}^{(v)} = (\bm{\mathcal {\overline{S}}}^{(v)})^{\mathrm{T}} \bm{\mathcal {\overline{Q}}}^{(v)}
\end{equation}

Spell the $\bm{\mathcal {\overline{G}}}^{(v)}$ into $\bm{\mathcal {\overline{G}}}$ and operate an inverse Fourier transform to get $\bm{\mathcal G}$, which is the solution of  (\ref{solveG1}).

$\bullet$\textbf{Solve $\bm{\mathcal Q}$ with fixed $\bm{\mathcal H}, \bm{\mathcal G}, \bm{\mathcal J}$.}  (\ref{objective function}) becomes:
\begin{equation}\label{solveQ1}
    \min_{\bm{\mathcal Q}^{\mathrm{T}}\bm{\mathcal Q} = \bm{\mathcal I}} {\left\| \bm{\mathcal S} - \bm{\mathcal Q}\bm{\mathcal G}^{\mathrm{T}} \right \|}_F^2 + \frac{\mu}{2} {\left\| \bm{\mathcal H}-\bm{\mathcal Q} + \frac{\bm{\mathcal Y}_1}{\mu}\right \|}_F^2
\end{equation}

And (\ref{solveQ1}) is equivalent to the following in the frequency domain:
\begin{small}
\begin{equation}\label{solveQ2}
    \begin{aligned}
        \min_{(\bm{\mathcal {\overline{Q}}}^{(v)})^{\mathrm{T}} \bm{\mathcal {\overline{Q}}}^{(v)} = \mathbf I}
        &\sum_{v=1}^{V}{\left\| \bm{\mathcal {\overline{S}}}^{(v)} - \bm{\mathcal {\overline{Q}}}^{(v)} (\bm{\mathcal {\overline{G}}}^{(v)})^{\mathrm{T}} \right \|}_F^2  \\
        & + \sum_{v=1}^{V} \frac{\mu}{2} {\left\| \bm{\mathcal {\overline{H}}}^{(v)} - \bm{\mathcal {\overline{Q}}}^{(v)} + \frac{\bm{\mathcal {\overline{Y}}}_1^{(v)}}{\mu}\right \|}_F^2,
    \end{aligned}
\end{equation}
\end{small}
where $\bm{\mathcal {\overline{Q}}} = \mathrm{fft}({{\bm{\mathcal Q}}},[\ ],3)$, and the others in the same way.

And (\ref{solveQ2}) can be obviously rewritten as:
\begin{small}
\begin{equation}\label{solveQ3}
    \begin{aligned}
        &\min_{(\bm{\mathcal {\overline{Q}}}^{(v)})^{\mathrm{T}} \bm{\mathcal {\overline{Q}}}^{(v)} = \mathbf I}
        \mathrm{tr}\left( (\bm{\mathcal {\overline{S}}}^{(v)} - \bm{\mathcal {\overline{Q}}}^{(v)} (\bm{\mathcal {\overline{G}}}^{(v)})^{\mathrm{T}})^{\mathrm{T}} (\bm{\mathcal {\overline{S}}}^{(v)} - \bm{\mathcal {\overline{Q}}}^{(v)} (\bm{\mathcal {\overline{G}}}^{(v)})^{\mathrm{T}}) \right) \\
        & + \frac{\mu}{2} \mathrm{tr} \left( (\bm{\mathcal {\overline{H}}}^{(v)} + \frac{\bm{\mathcal {\overline{Y}}}^{(v)}_1}{\mu} - \bm{\mathcal {\overline{Q}}}^{(v)})^{\mathrm{T}} (\bm{\mathcal {\overline{H}}}^{(v)} + \frac{\bm{\mathcal {\overline{Y}}}^{(v)}_1}{\mu} - \bm{\mathcal {\overline{Q}}}^{(v)}) \right),
    \end{aligned}
\end{equation}
\end{small}
and it is apparent that
\begin{small}
\begin{equation}\label{solveQ4}
    \min_{(\bm{\mathcal {\overline{Q}}}^{(v)})^{\mathrm{T}} \bm{\mathcal {\overline{Q}}}^{(v)} = \mathbf I}
    -2\mathrm{tr} \left( (\bm{\mathcal {\overline{Q}}}^{(v)})^{\mathrm{T}} \bm{\mathcal {\overline{S}}}^{(v)} \bm{\mathcal {\overline{G}}}^{(v)} \right) - \mu \mathrm{tr} \left( (\bm{\mathcal {\overline{Q}}}^{(v)})^{\mathrm{T}} \bm{\mathcal {\overline{W}}}_1^{(v)} \right)
\end{equation}
\end{small}
where $\bm{\mathcal {\overline{W}}}_1^{(v)} =  \bm{\mathcal {\overline{H}}}^{(v)} + \frac{ \bm{\mathcal {\overline{Y}}}^{(v)}_1}{\mu}$.

Moreover, (\ref{solveQ4}) can be rewritten in its equivalent form:
\begin{small}
\begin{equation}\label{solveQ5}
    \max_{(\bm{\mathcal {\overline{Q}}}^{(v)})^{\mathrm{T}} \bm{\mathcal {\overline{Q}}}^{(v)} = \mathbf I}
    2\mathrm{tr} \left( (\bm{\mathcal {\overline{Q}}}^{(v)})^{\mathrm{T}} \bm{\mathcal {\overline{S}}}^{(v)} \bm{\mathcal {\overline{G}}}^{(v)} \right) + \mu \mathrm{tr} \left( (\bm{\mathcal {\overline{Q}}}^{(v)})^{\mathrm{T}} \bm{\mathcal {\overline{W}}}_1^{(v)} \right),
\end{equation}
\end{small}
and it also can be reduced to:
\begin{equation}\label{solveQ6}
    \max_{(\bm{\mathcal {\overline{Q}}}^{(v)})^{\mathrm{T}} \bm{\mathcal {\overline{Q}}}^{(v)} = \mathbf I}
    \mathrm{tr} \left( (\bm{\mathcal {\overline{Q}}}^{(v)})^{\mathrm{T}} \bm{\mathcal {\overline{B}}}^{(v)} \right)
\end{equation}
where $\bm{\mathcal {\overline{B}}}^{(v)} = 2 \bm{\mathcal {\overline{S}}}^{(v)} \bm{\mathcal {\overline{G}}}^{(v)} + \mu \bm{\mathcal {\overline{W}}}_1^{(v)}$.

To solve (\ref{solveQ6}), we introduce the following Theorem:
\begin{theorem}\label{theorem solveF}
Given $\mathbf G$ and $\mathbf P$, where $\mathbf G (\mathbf G)^{\mathrm{T}}=\mathbf I$ and $\mathbf P$ has the singular value decomposition $\mathbf P=\mathbf \Lambda \mathbf S(\mathbf V)^{\mathrm{T}}$, then the optimal solution of
\begin{equation}\label{the1}
\max_{\mathbf G (\mathbf G)^{\mathrm{T}}=\mathbf I} \mathrm{tr}( \mathbf G \mathbf P)
\end{equation}
is $\mathbf G^\ast=\mathbf V[\mathbf I,\mathbf 0](\mathbf \Lambda)^{\mathrm{T}}$.
\end{theorem}

\begin{proof}
From the SVD $\mathbf P=\mathbf \Lambda \mathbf S (\mathbf V)^{\mathrm{T}}$ and together with  (\ref{the1}), it is evident that
\begin{equation}\label{solveF5}
\begin{aligned}
		\mathrm{tr}(\mathbf G \mathbf P ) &= \mathrm{tr}(\mathbf G\mathbf \Lambda^{(v)}\mathbf S(\mathbf V)^{\mathrm{T}}  ) \\
		&=\mathrm{tr}(\mathbf S (\mathbf V)^{\mathrm{T}} \mathbf G \mathbf \Lambda  ) \\
		&=\mathrm{tr}(\mathbf S \mathbf H )\\
		&=\sum_i s_{ii}  h_{ii}
	,\end{aligned}
\end{equation}
where $\mathbf H=(\mathbf V)^{\mathrm{T}} \mathbf G \mathbf \Lambda$, $s_{ii}$ and $h_{ii}$ are the $(i,i)$ elements of $\mathbf S$ and $\mathbf H$, respectively. It can be easily verified that $\mathbf H (\mathbf H)^{\mathrm{T}}=\mathbf I$, where $\mathbf I$ is an identity matrix. Therefore $-1\leqslant h_{ii} \leqslant 1$ and $s_{ii} \geqslant 0$, Thus we have:
\begin{equation}\label{solveF6}
\mathrm{tr}(\mathbf G \mathbf P)=\sum_i s_{ii} h_{ii} \leqslant \sum_i s_{ii}.
\end{equation}
The equality holds when $\mathbf H$ is an identity matrix. So $\mathrm{tr}(\mathbf G \mathbf P)$ reaches the maximum when $\mathbf H = [\mathbf I,\mathbf 0]$.
\end{proof}

So the solution of  (\ref{solveQ6}) is:
\begin{equation}\label{solveQ}
    \bm{\mathcal {\overline{Q}}}^{(v)} = \bm{{\overline{\Lambda}}}^{(v)} (\bm{{\overline{V}}}^{(v)})^{\mathrm{T}}
\end{equation}

Spell the $\bm{\mathcal {\overline{Q}}}^{(v)}$ into $\bm{\mathcal {\overline{Q}}}$ and operate an inverse Fourier transform to get $\bm{\mathcal Q}$, which is the solution of  (\ref{solveQ1}).

$\bullet$\textbf{Solve $\bm{\mathcal H}$ with fixed $\bm{\mathcal Q}, \bm{\mathcal G}, \bm{\mathcal J}$.}  (\ref{objective function}) becomes:
\begin{equation}\label{solveH1}
	\begin{aligned}
		\min_{\bm{\mathcal H}\geqslant 0} \frac{\mu}{2} {\left\| \bm{\mathcal H}-\bm{\mathcal Q} + \frac{\bm{\mathcal Y}_1}{\mu}\right \|}_F^2 + \frac{\rho}{2} {\left\| \bm{\mathcal H}-\bm{\mathcal J} + \frac{\bm{\mathcal Y}_2}{\rho}\right \|}_F^2
	\end{aligned}
\end{equation}


 (\ref{solveH1}) is obviously equivalent to:
\begin{equation}\label{solveH2}
	\begin{aligned}
		\min_{\bm{\mathcal H}\geqslant 0}
        (\frac{\mu + \rho}{2}) {\left\| \bm{\mathcal {H}} - \frac{\mu \bm{\mathcal {W}}_2 + \rho \bm{\mathcal {W}}_3}{\mu + \rho} \right \|}_F^2
	\end{aligned}
\end{equation}
where $\bm{\mathcal {W}}_2 =  \bm{\mathcal {Q}} - \frac{ \bm{\mathcal {Y}}_1}{\mu}$ and $\bm{\mathcal {W}}_3 = \bm{\mathcal {J}} - \frac{ \bm{\mathcal {Y}}_2}{\rho}$.

Then the solution of  (\ref{solveH2}) is:
\begin{equation}\label{solveH}
    \bm{\mathcal {H}} = \left(\frac{\mu \bm{\mathcal {W}}_2 + \rho \bm{\mathcal {W}}_3}{\mu + \rho} \right)_+
\end{equation}


$\bullet$\textbf{Solve $\bm{\mathcal J}$ with fixed $\bm{\mathcal Q}, \bm{\mathcal H}, \bm{\mathcal G}$.}  (\ref{objective function}) becomes:
\begin{equation}\label{solveJ1}
	\begin{aligned}
		\min \lambda{\| \bm{\mathcal J}\|}_\Sp^p + \frac{\rho}{2} {\left\| \bm{\mathcal H}-\bm{\mathcal J} + \frac{\bm{\mathcal Y}_2}{\rho}\right \|}_F^2,
	\end{aligned}
\end{equation}
after completing the square regarding $\bm{\mathcal J}$, we can deduce
\begin{equation}\label{solve J}
	\begin{aligned}
		\bm{\mathcal J}^* = \arg \min \frac{1}{2}\left\|{\bm{\mathcal H} + \frac{\bm{\mathcal Y}_2}{\rho} - \bm{\mathcal J}}\right\|_F^2 + \frac{\lambda}{\rho}{\|\bm{\mathcal J}\|}_\Sp^p,
	\end{aligned}
\end{equation}
which has a closed-form solution as Lemma \ref{T2} \cite{gao2020enhanced}:

\begin{lemma}\label{T2}
    Let ${\mathcal Z} \in {\mathbb{R}}{^{{n_1} \times {n_2} \times {n_3}}}$ have a t-SVD ${\mathcal Z} = {\mathcal U} * {\mathcal S} * {{\mathcal V}^{\mathrm{T}}}$, then the optimal solution for
    \begin{equation}\label{tensor-gaozx-2020}
        \begin{array}{l}
            \mathop {\min }\limits_{\mathcal X} \frac{1}{2}\left\| {{\mathcal X} - {\mathcal Z}} \right\|_F^2 + \tau \left\| {\mathcal X} \right\|_{{\Sp}}^p.
        \end{array}
    \end{equation}
is ${{\cal X}^*} = {\Gamma _\tau }({\cal Z}) = {\cal U}*\mathrm{ifft}({P_\tau }(\overline {\cal Z} ))*{{\cal V}^{\mathrm{T}}}$, where ${P_\tau }(\overline {\cal Z} )$ is an f-diagonal  3rd-order tensor, whose diagonal elements can be found by using the GST algorithm introduced in \cite{gao2020enhanced}.
\end{lemma}

Now the solution of  (\ref{solve J}) is:
\begin{equation}\label{21}
    \bm{\mathcal J}^* = {\Gamma _{\frac{\lambda}{\rho}}} (\bm{\mathcal H} + \frac{\bm{\mathcal Y}_2}{\rho}).
\end{equation}

Finally, the optimization procedure for Multi-View Clustering via Semi-non-negative Tensor Factorization (Semi-NTF) is outlined in Algorithm 1.

\begin{algorithm}[tb]
\caption{Multi-View Clustering via Semi-non-negative Tensor Factorization (Semi-NTF)}
\label{A1}
\textbf{Input}: Data matrices $\{{\mathbf{X}}^{(v)}\}_{v=1}^{V}\in \mathbb{R}^{N\times d_v}$; anchors numbers $m$; cluster number $K$.\\
\textbf{Output}: Cluster labels $\mathbf{Y}$ of each data points.\\
\textbf{Initialize}: $\mu=10^{-5}$, $\rho=10^{-5}$, $\eta=1.6$, $\bm{\mathcal Y}_1=0$, $\bm{\mathcal Y}_2=0$, $\mathbf{\overline{Q}}^{(v)}$ is identity matrix;
\begin{algorithmic}[1] 
\STATE Compute graph matrix $\mathbf S^{(v)}$ of each views;
\WHILE{not condition}
\STATE Update $\bm{\mathcal {\overline{G}}}^{(v)}$ by solving  (\ref{solveG});
\STATE Update $\bm{\mathcal {\overline{Q}}}^{(v)}$ by solving  (\ref{solveQ});
\STATE Update $\bm{\mathcal {\overline{H}}}^{(v)}$ by solving  (\ref{solveH});
\STATE Update ${\bm{{\mathcal J}}}$ by using  (\ref{solve J});
\STATE Update $\bm{\mathcal Y}_1$, $\bm{\mathcal Y}_2$, $\mu$ and $\rho$: $\bm{\mathcal Y}_1=\bm{\mathcal Y}_1+\mu(\bm{\mathcal H}-\bm{\mathcal Q})$, $\bm{\mathcal Y}_2=\bm{\mathcal Y}_2+\mu(\bm{\mathcal H}-\bm{\mathcal J})$, $\mu=\eta\mu$, $\rho=\eta\rho$;
\ENDWHILE
\STATE Calculate the $K$ clusters by using \\
$\mathbf H=\sum_{v=1}^V \mathbf H^{(v)} / V$;
\STATE \textbf{return} Clustering result.
\end{algorithmic}
\end{algorithm}

\subsection{Convergence Analysis}
\begin{theorem}\label{thm1}[Convergence Analysis of Algorithm~\ref{A1}]
Let $\mathcal{P}_{k}=\{\bm{{\mathcal{Q}}}_{k},\bm{{\mathcal{H}}}_{k}, \bm{{\mathcal{G}}}_{k}, \bm{{\mathcal{J}}}_{k}, \bm{{\mathcal{Y}}}_{2,k}, \bm{{\mathcal{Y}}}_{1,k}\},\ 1\leq k< \infty$ in \eqref{objective function} be a sequence generated
by \textbf{Algorithm~1}, then
\begin{enumerate}
\item  $\mathcal{P}_{k}$ is bounded with the assumption $\lim_{k\rightarrow 0}\max\{\mu_k,\rho_k\}({\bar{\mathcal H}_{k+1}^{(v)}}-{\bar{\mathcal H}_k^{(v)}})=0$;
\item  Any accumulation point of $\mathcal{P}_{k}$ is a stationary KKT point of \eqref{objective function}.
\end{enumerate}
\end{theorem}
The proof will be provided in the appendix and we need to mention that the KKT conditions can be used to determine the stop conditions for Algorithm \ref{A1}, which are
$\| \bm{{\mathcal{Q}}}_{k}-\bm{{\mathcal{H}}}_{k} \|_\infty\leq\varepsilon$, $\| \bm{{\mathcal{Q}}}_{k}-\bm{{\mathcal{J}}}_{k} \|_\infty\leq\varepsilon$.

\begin{table*}[t]
\caption{Clustering performance on Mnist4 and NUS. (The best result is in bold, and the second-best result is underlined.)}
\label{clusteringResultOnMnist4NUS}
\centering
\scalebox{0.7}
{
\begin{tabular}{ccccccccc}
\toprule
Dataset&\multicolumn{7}{c}{\textbf{Mnist4}}\\
\midrule
Metrics &ACC &NMI &Purity &PER &REC &F-score &ARI\\
\midrule
AMGL~\cite{nie2016parameter} &\underline{0.921$\pm$0.000} &0.806$\pm$0.000 &\underline{0.921$\pm$0.000} &\underline{0.854$\pm$0.000} &\underline{0.862$\pm$0.000} &\underline{0.858$\pm$0.000} &\underline{0.810$\pm$0.000} \\
MVGL~\cite{zhan2017graph} &0.919$\pm$0.000 &0.803$\pm$0.000 &0.919$\pm$0.000 &0.851$\pm$0.000 &0.860$\pm$0.000 &0.856$\pm$0.000 &0.807$\pm$0.000 \\
CSMSC~\cite{luo2018consistent} &0.641$\pm$0.000 &0.601$\pm$0.000 &0.728$\pm$0.000 &0.607$\pm$0.000 &0.767$\pm$0.000 &0.677$\pm$0.000 &0.553$\pm$0.000 \\
GMC~\cite{wang2019gmc} &0.920$\pm$0.000 &\underline{0.807$\pm$0.000} &0.920$\pm$0.000 &0.853$\pm$0.000 &0.861$\pm$0.000 &0.857$\pm$0.000 &0.809$\pm$0.000 \\
LMVSC~\cite{kang2020large} &0.892$\pm$0.000 &0.726$\pm$0.000 &0.892$\pm$0.000 &0.808$\pm$0.000 &0.812$\pm$0.000 &0.810$\pm$0.000 &0.747$\pm$0.000 \\
SMSC~\cite{hu2020multi} &0.909$\pm$0.000 &0.774$\pm$0.000 &0.909$\pm$0.000 &0.834$\pm$0.000 &0.841$\pm$0.000 &0.837$\pm$0.000 &0.783$\pm$0.000 \\
SFMC~\cite{li2020multiview} &0.916$\pm$0.000 &0.797$\pm$0.000 &0.916$\pm$0.000 &0.846$\pm$0.000 &0.855$\pm$0.000 &0.850$\pm$0.000 &0.800$\pm$0.000 \\
FMCNOF~\cite{yang2021fast} &0.663$\pm$0.024 &0.455$\pm$0.033 &0.654$\pm$0.016 &0.510$\pm$0.036 &0.677$\pm$0.071 &0.580$\pm$0.034 &0.413$\pm$0.047 \\
Semi-NTF &\textbf{0.990$\pm$0.000} &\textbf{0.962$\pm$0.000} &\textbf{0.990$\pm$0.000} &\textbf{0.981$\pm$0.000} &\textbf{0.981$\pm$0.000} &\textbf{0.981$\pm$0.000} &\textbf{0.975$\pm$0.000} \\
\midrule
Dataset&\multicolumn{7}{c}{\textbf{NUS}}\\
\midrule
Metrics &ACC &NMI &Purity &PER &REC &F-score &ARI\\
\midrule
AMGL~\cite{nie2016parameter} &0.214$\pm$0.000 &0.121$\pm$0.000 &0.232$\pm$0.000 &0.103$\pm$0.000 &0.370$\pm$0.000 &0.161$\pm$0.000 &0.036$\pm$0.000 \\
MVGL~\cite{zhan2017graph} &0.145$\pm$0.000 &0.067$\pm$0.000 &0.153$\pm$0.000 &0.085$\pm$0.000 &\textbf{0.851$\pm$0.000} &0.155$\pm$0.000 &0.005$\pm$0.000 \\
CSMSC~\cite{luo2018consistent} &0.224$\pm$0.006 &0.115$\pm$0.002 &0.235$\pm$0.006 &0.137$\pm$0.002 &0.152$\pm$0.005 &0.144$\pm$0.003 &0.063$\pm$0.002 \\
GMC~\cite{wang2019gmc} &0.165$\pm$0.000 &0.078$\pm$0.000 &0.178$\pm$0.000 &0.088$\pm$0.000 &0.764$\pm$0.000 &0.159$\pm$0.000 &0.012$\pm$0.000 \\
LMVSC~\cite{kang2020large} &0.251$\pm$0.000 &0.129$\pm$0.000 &0.268$\pm$0.000 &0.128$\pm$0.000 &0.167$\pm$0.000 &0.145$\pm$0.000 &0.057$\pm$0.000 \\
SMSC~\cite{hu2020multi} &\underline{0.297$\pm$0.000} &\underline{0.165$\pm$0.000} &\underline{0.324$\pm$0.000} &\underline{0.166$\pm$0.000} &0.199$\pm$0.000 &\underline{0.181$\pm$0.000} &\underline{0.100$\pm$0.000} \\
SFMC~\cite{li2020multiview} &0.134$\pm$0.000 &0.047$\pm$0.000 &0.139$\pm$0.000 &0.084$\pm$0.000 &\underline{0.796$\pm$0.000} &0.152$\pm$0.000 &0.002$\pm$0.000 \\
FMCNOF~\cite{yang2021fast} &0.172$\pm$0.035 &0.065$\pm$0.016 &0.177$\pm$0.034 &0.100$\pm$0.009 &0.457$\pm$0.142 &0.161$\pm$0.004 &0.032$\pm$0.014 \\
Semi-NTF &\textbf{0.659$\pm$0.000} &\textbf{0.674$\pm$0.000} &\textbf{0.675$\pm$0.000} &\textbf{0.526$\pm$0.000} &0.553$\pm$0.000 &\textbf{0.539$\pm$0.000} &\textbf{0.496$\pm$0.000} \\
\bottomrule
\end{tabular}
}
\end{table*}

\begin{table*}[t]
\caption{Clustering performance on AWA and Scene-15. (The best result is in bold, and the second-best result is underlined. )}
\label{clusteringResultOnAWAScene15}
\centering
\scalebox{0.7}
{
\begin{tabular}{ccccccccc}
\toprule
Dataset &\multicolumn{7}{c}{\textbf{AWA}}\\
\midrule
Metrics &ACC &NMI &Purity &PER &REC &F-score &ARI \\
\midrule
MVGL~\cite{zhan2017graph} &0.061$\pm$0.000 &0.070$\pm$0.000 &0.065$\pm$0.000 &0.020$\pm$0.000 &\underline{0.843$\pm$0.000} &0.040$\pm$0.000 &0.002$\pm$0.000 \\
CSMSC~\cite{luo2018consistent} &\underline{0.113$\pm$0.000} &\underline{0.175$\pm$0.002} &\underline{0.119$\pm$0.001} &\underline{0.051$\pm$0.001} &0.054$\pm$0.000 &\underline{0.053$\pm$0.001} &\underline{0.033$\pm$0.001} \\
GMC~\cite{wang2019gmc} &0.028$\pm$0.000 &0.030$\pm$0.000 &0.039$\pm$0.000 &0.020$\pm$0.000 &\textbf{0.915$\pm$0.000} &0.039$\pm$0.000 &0.001$\pm$0.000 \\
LMVSC~\cite{kang2020large} &0.105$\pm$0.000 &0.171$\pm$0.000 &0.114$\pm$0.000 &0.041$\pm$0.000 &0.065$\pm$0.000 &0.051$\pm$0.000 &0.027$\pm$0.000 \\
SFMC~\cite{li2020multiview} &0.042$\pm$0.000 &0.044$\pm$0.000 &0.049$\pm$0.000 &0.023$\pm$0.000 &0.592$\pm$0.000 &0.044$\pm$0.000 &0.006$\pm$0.000 \\
FMCNOF~\cite{yang2021fast} &0.035$\pm$0.008 &0.018$\pm$0.011 &0.035$\pm$0.008 &0.021$\pm$0.001 &0.688$\pm$0.227 &0.040$\pm$0.002 &0.002$\pm$0.003 \\
Semi-NTF &\textbf{0.512$\pm$0.000} &\textbf{0.723$\pm$0.000} &\textbf{0.538$\pm$0.000} &\textbf{0.407$\pm$0.000} &0.434$\pm$0.000 &\textbf{0.420$\pm$0.000} &\textbf{0.408$\pm$0.000} \\
\midrule
Dataset&\multicolumn{7}{c}{\textbf{Scene-15}}\\
\midrule
Metrics &ACC &NMI &Purity &PER &REC &F-score &ARI \\
\midrule
AMGL~\cite{nie2016parameter} &0.332$\pm$0.000 &0.303$\pm$0.000 &0.340$\pm$0.000 &0.169$\pm$0.000 &0.375$\pm$0.000 &0.233$\pm$0.000 &0.152$\pm$0.000 \\
MVGL~\cite{zhan2017graph} &0.183$\pm$0.000 &0.154$\pm$0.000 &0.204$\pm$0.000 &0.083$\pm$0.000 &\underline{0.753$\pm$0.000} &0.150$\pm$0.000 &0.030$\pm$0.000 \\
CSMSC~\cite{luo2018consistent} &0.334$\pm$0.008 &0.313$\pm$0.005 &0.378$\pm$0.003 &0.227$\pm$0.003 &0.239$\pm$0.002 &0.233$\pm$0.003 &0.174$\pm$0.003 \\
GMC~\cite{wang2019gmc} &0.140$\pm$0.000 &0.058$\pm$0.000 &0.146$\pm$0.000 &0.071$\pm$0.000 &\textbf{0.893$\pm$0.000} &0.131$\pm$0.000 &0.004$\pm$0.000 \\
LMVSC~\cite{kang2020large} &0.355$\pm$0.000 &0.331$\pm$0.000 &0.399$\pm$0.000 &0.238$\pm$0.000 &0.244$\pm$0.000 &0.241$\pm$0.000 &0.184$\pm$0.000 \\
SMSC~\cite{hu2020multi} &\underline{0.422$\pm$0.000} &\underline{0.392$\pm$0.000} &\underline{0.475$\pm$0.000} &\underline{0.269$\pm$0.000} &0.333$\pm$0.000 &\underline{0.298$\pm$0.000} &\underline{0.240$\pm$0.000} \\
SFMC~\cite{li2020multiview} &0.188$\pm$0.000 &0.135$\pm$0.000 &0.202$\pm$0.000 &0.087$\pm$0.000 &0.344$\pm$0.000 &0.139$\pm$0.000 &0.032$\pm$0.000 \\
FMCNOF~\cite{yang2021fast} &0.218$\pm$0.033 &0.166$\pm$0.022 &0.221$\pm$0.029 &0.117$\pm$0.016 &0.471$\pm$0.062 &0.186$\pm$0.019 &0.085$\pm$0.026 \\
Semi-NTF &\textbf{0.758$\pm$0.000} &\textbf{0.804$\pm$0.000} &\textbf{0.759$\pm$0.000} &\textbf{0.642$\pm$0.000} &0.656$\pm$0.000 &\textbf{0.649$\pm$0.000} &\textbf{0.622$\pm$0.000} \\
\bottomrule
\end{tabular}
}
\end{table*}

\begin{table*}[t]
\caption{Clustering results and running time on Reuters. (The best result is in bold, and the second-best result is underlined. ``OM" means out of memory and ``-" means the algorithm takes more than three hours to calculate.)}
\label{times}
\centering
\scalebox{0.75}
{
\begin{tabular}{ccccc}
\toprule
Dataset &\multicolumn{4}{c}{\textbf{Reuters}}\\
\midrule
Metrics &ACC &NMI &Purity &Running Time (in seconds) \\
\midrule
AMGL~\cite{nie2016parameter} &OM &OM &OM &OM \\
MVGL~\cite{zhan2017graph} &OM &OM &OM &OM   \\
CSMSC~\cite{luo2018consistent} &OM &OM &OM &OM  \\
GMC~\cite{wang2019gmc} &- &- &- &- \\
LMVSC~\cite{kang2020large} &0.587$\pm$0.000 &0.335$\pm$0.000 &\underline{0.616$\pm$0.000} &\underline{16.91}   \\
SMSC~\cite{hu2020multi} &OM &OM &OM &OM  \\
SFMC~\cite{li2020multiview} &\underline{0.602$\pm$0.000} &\underline{0.354$\pm$0.000} &0.552$\pm$0.000 &31.00 \\
FMCNOF~\cite{yang2021fast} &0.343$\pm$0.007 &0.125$\pm$0.037 &0.358$\pm$0.052 &\textbf{10.05} \\
Semi-NTF &\textbf{0.795$\pm$0.000} &\textbf{0.664$\pm$0.000} &\textbf{0.841$\pm$0.000} &419.99 \\
\bottomrule
\end{tabular}
}
\end{table*}

\section{Experiments}\label{Experiment}
In this section, we demonstrate the performance of our proposed method through extensive experiments. It is compared with plenty of state-of-art multi-view clustering algorithms on 5 widely used multi-view datasets. We evaluate the clustering performance by applying 7 metrics used widely, \ie, 1) ACC; 2) NMI; 3) Purity; 4) PRE; 5) REC; 6) F-score; and 7) ARI. The higher the value the better the clustering results for all metrics mentioned above. The experiments are implemented on a standard Windows 10 Server with two Intel (R) Xeon (R) Gold 6230 CPUs and 128 GB RAM, MATLAB R2020a.

\subsection{Datasets and Compared Baselines Methods}
The following 5 multi-view datasets are selected to examine our proposed method.
\begin{itemize}
  \item Mnist4 is from Mnist\cite{lecun1998gradient} dataset. Mnist has a training set of 10000 examples, and the first four categories(0-3) of it, \ie, the first 4000 samples form a subset called Mnist4. It is divided into 4 categories and described by 3 different features:  30 Isometric Projection (IP), 9 Linear Discriminant Analysis (LDA) and 30 Neighbourhood Preserving Embedding (NPE).
  \item AWA is a dataset of animal images, which consists of 30475 images of 50 animals categories. 6 features are used to describe each image, \ie, 2000 Color Histogram (CH), 2000 Local Self-Similarities (LSS), 252 Pyramid HOG (PHOG), 2000 SIFT, 2000 color SIFT (CSIFT), 2000 SURF. We take out 80 images from each category randomly obtaining a dataset containing a total of 4000 images~\cite{fu2015transductive}.
  \item NUS derive from NUS\_WIDE\_Object\cite{chua2009nus} containing 30000 images in 31 classes. The first 12 categories and the first 200 images of each category consist NUS dataset. Its features include  64 Color Histogram (CH), 144 Color Correlation (CC), 73 Edge Direction Histograms (EDH), 128 Wavelet Texture (WT), 225 Block-Wise Color Moment (BWCM), 500 Bag of Words based on SIFT (BWSIFT).
  \item Scene-15~\cite{oliva2001modeling,fei2005bayesian,lazebnik2006beyond} is a dataset of indoor and outdoor environment images, which consists of 4485 images in total and is divided into 15 categories. It is described by 3 features.
  \item Reuters~\cite{li2015large} is a dataset of documents, which consisted of 18758 samples of six classes. Each sample is written by five languages. Different languages are considered as different views. These languages are 21531 English, 24892 French, 34251 German, 15506 Italian, and 11547 Spanish.
\end{itemize}

We choose the following 8 state-of-art multi-view clustering algorithms to compare with our proposed methods:
\textbf{AMGL}~\cite{nie2016parameter};
\textbf{MVGL}~\cite{zhan2017graph};
\textbf{CSMSC}~\cite{luo2018consistent};
\textbf{GMC}~\cite{wang2019gmc};
\textbf{LMVSC}~\cite{kang2020large};
\textbf{SMSC}~\cite{hu2020multi};
\textbf{SFMC}~\cite{li2020multiview}
\textbf{FMCNOF}\cite{yang2021fast};


\subsection{Experiments Result}
The clustering results of several state-of-art multi-view clustering algorithms and our proposed approach on Mnist4, NUS, AWA, and Scene-15 are listed in Table \ref{clusteringResultOnMnist4NUS} and Table \ref{clusteringResultOnAWAScene15}, respectively. In order to demonstrate the advantages of our algorithm on large-scale datasets, we conducted some experiments on partly clustering metrics and running time on Reuters and they are shown as in Table \ref{times}.
To reduce the influence of random initialization and other factors on clustering results, we conducted 20 times independent experiments for each comparison algorithm. Finally, we took their average and standard deviations as the final performance.

It is clear that our algorithm outperforms the other baseline algorithms on most of the datasets and maintains a similar level of best results to the other algorithms on a small number of datasets according to Table \ref{clusteringResultOnMnist4NUS}, \ref{clusteringResultOnAWAScene15} and \ref{times}. We did not show AMGL and SMSC in Table \ref{clusteringResultOnAWAScene15} because of the error when running the code provided by their authors in AWA.


\subsection{Impact for Parameters}
In this subsection, we analyze the effect of variable parameters on 4 datasets.

\begin{figure}[t]
	\centering
	\includegraphics[width=0.75\linewidth]{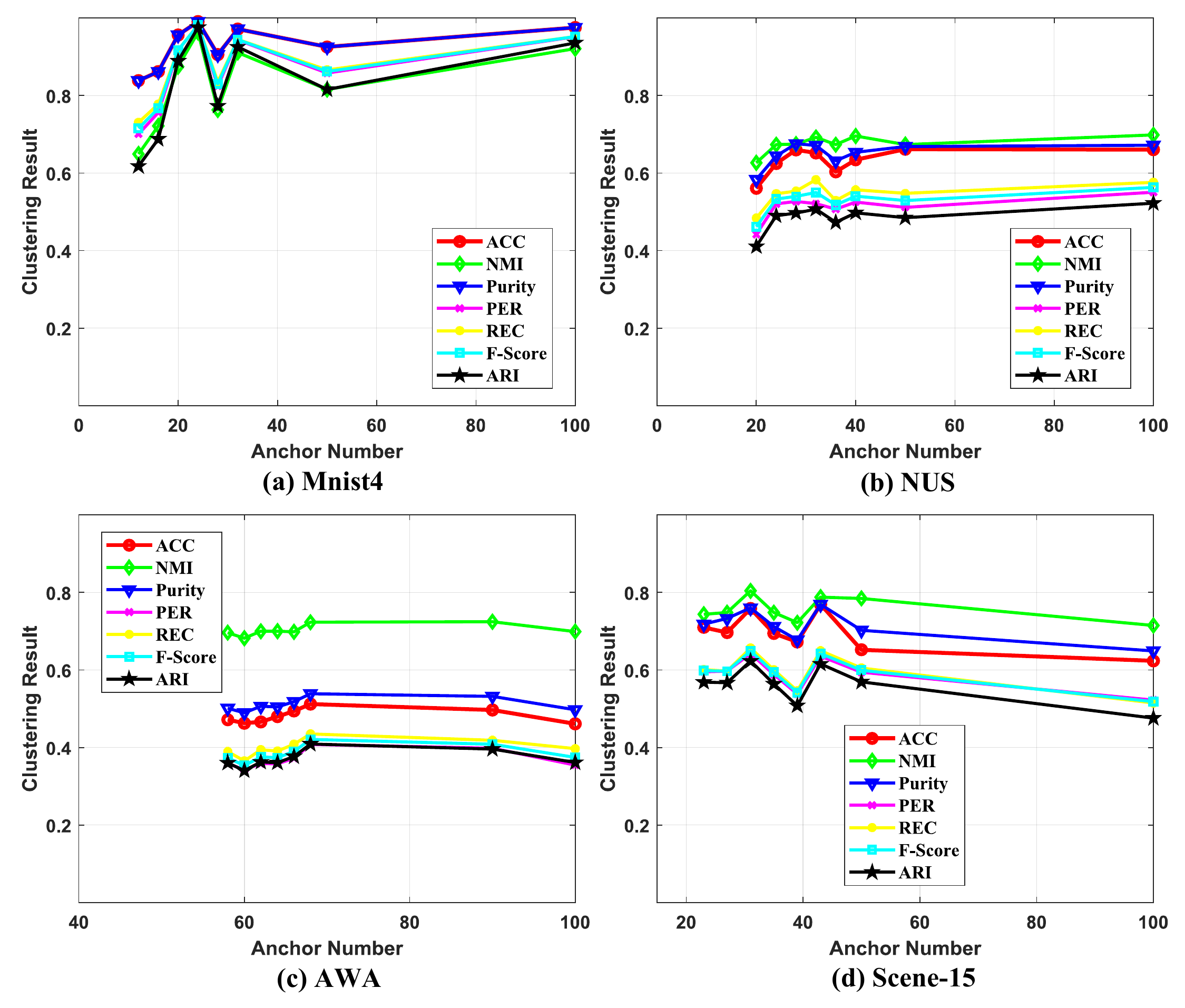}
	\caption{Clustering results with different number of anchors on Mnist4, NUS, AWA and Scene-15 datasets.}
	\label{fig_sim}
    \vspace{-5mm}
\end{figure}

\begin{figure}[t]
	\centering
	\includegraphics[width=0.75\linewidth]{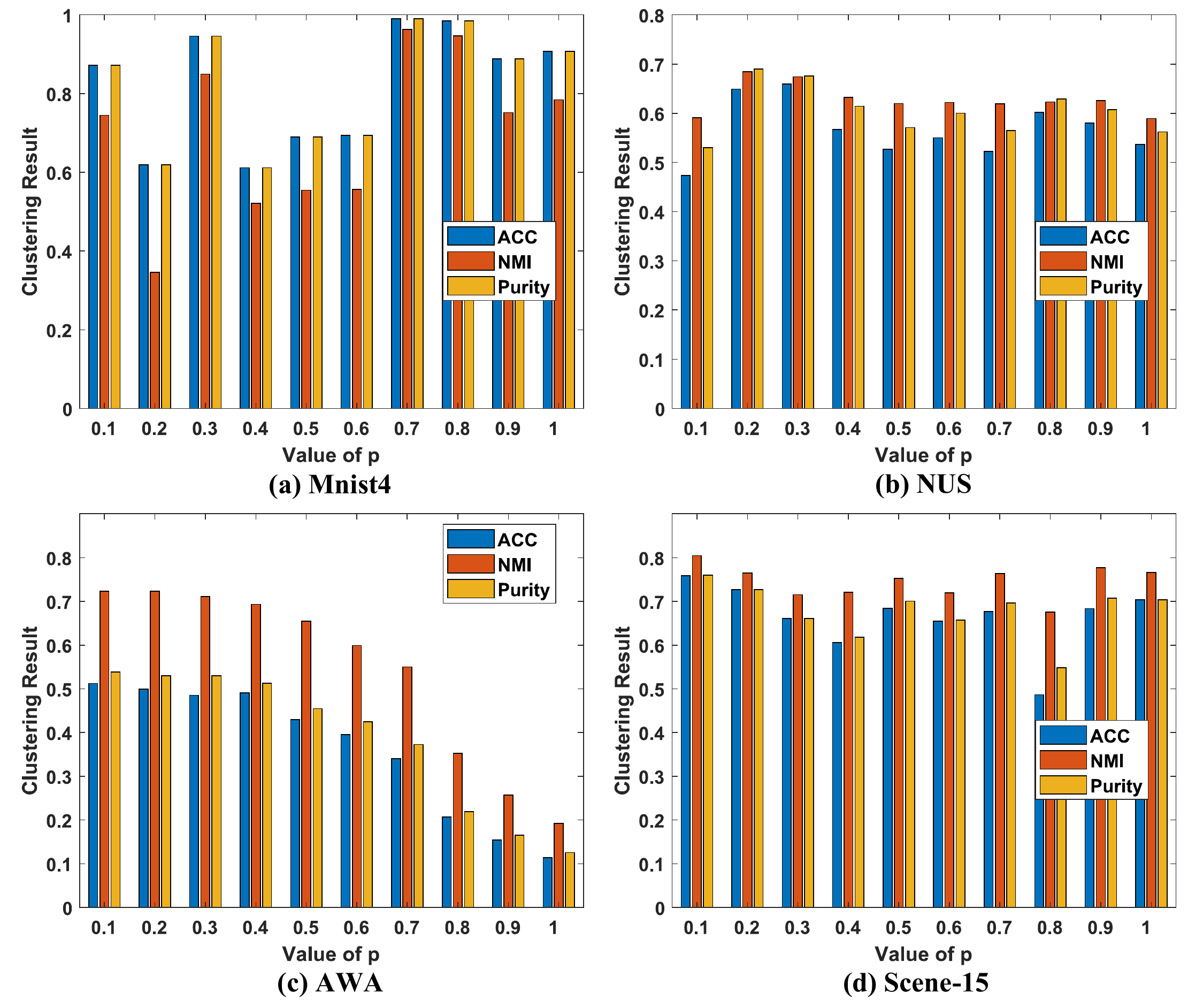}
	\caption{The influence of p on clustering results on Mnist4, NUS, AWA and Scene-15 datasets.}
	\label{fig_sim2}
    \vspace{-5mm}
\end{figure}

\begin{figure}[t]
	\centering
	\includegraphics[width=0.75\linewidth]{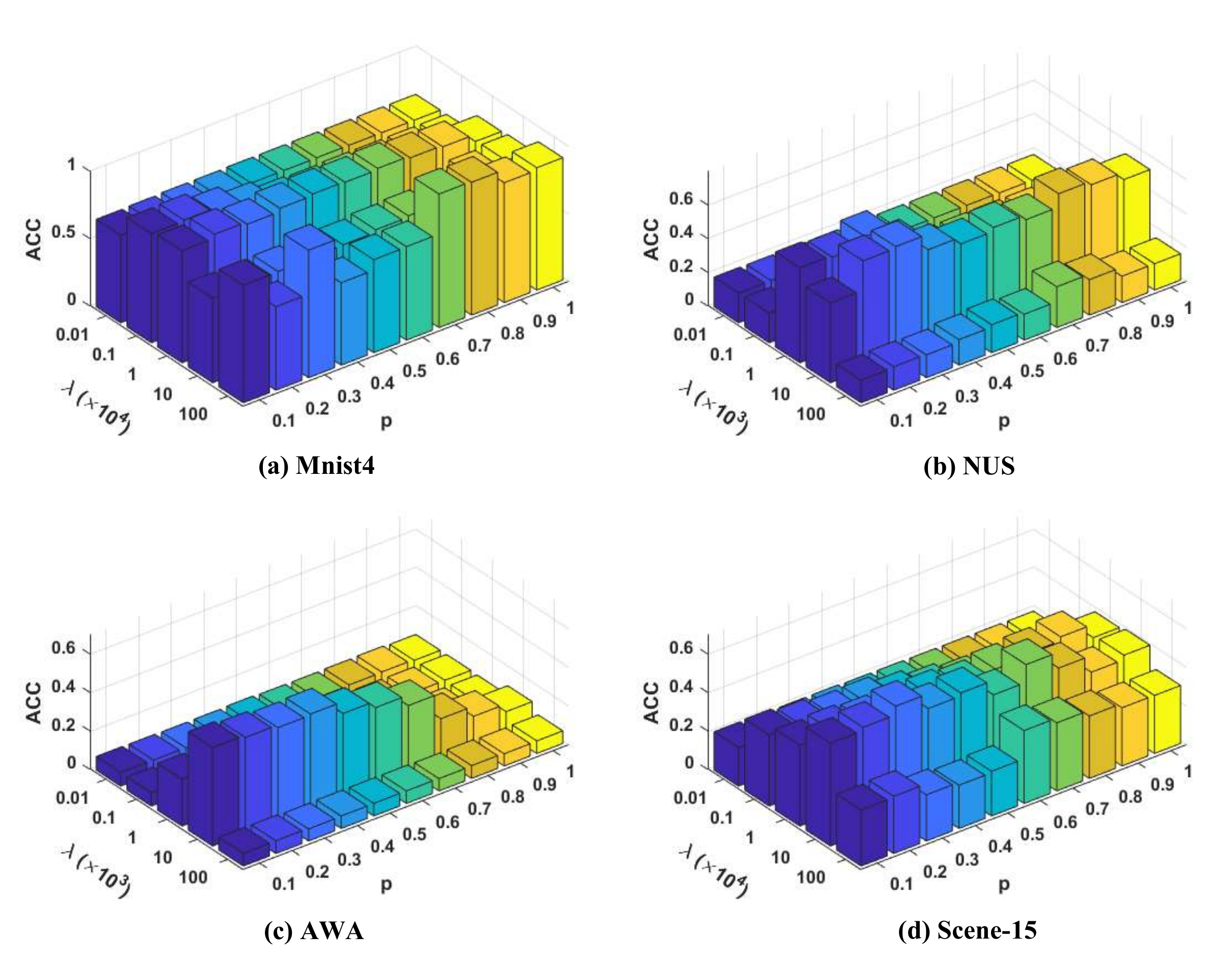}
	\caption{The influence of $\lambda$ on ACC on Mnist4, NUS, AWA and Scene-15 datasets.}
	\label{fig_sim3}
    \vspace{-5mm}
\end{figure}

\begin{figure}[t]
	\centering
	\includegraphics[width=0.7\linewidth]{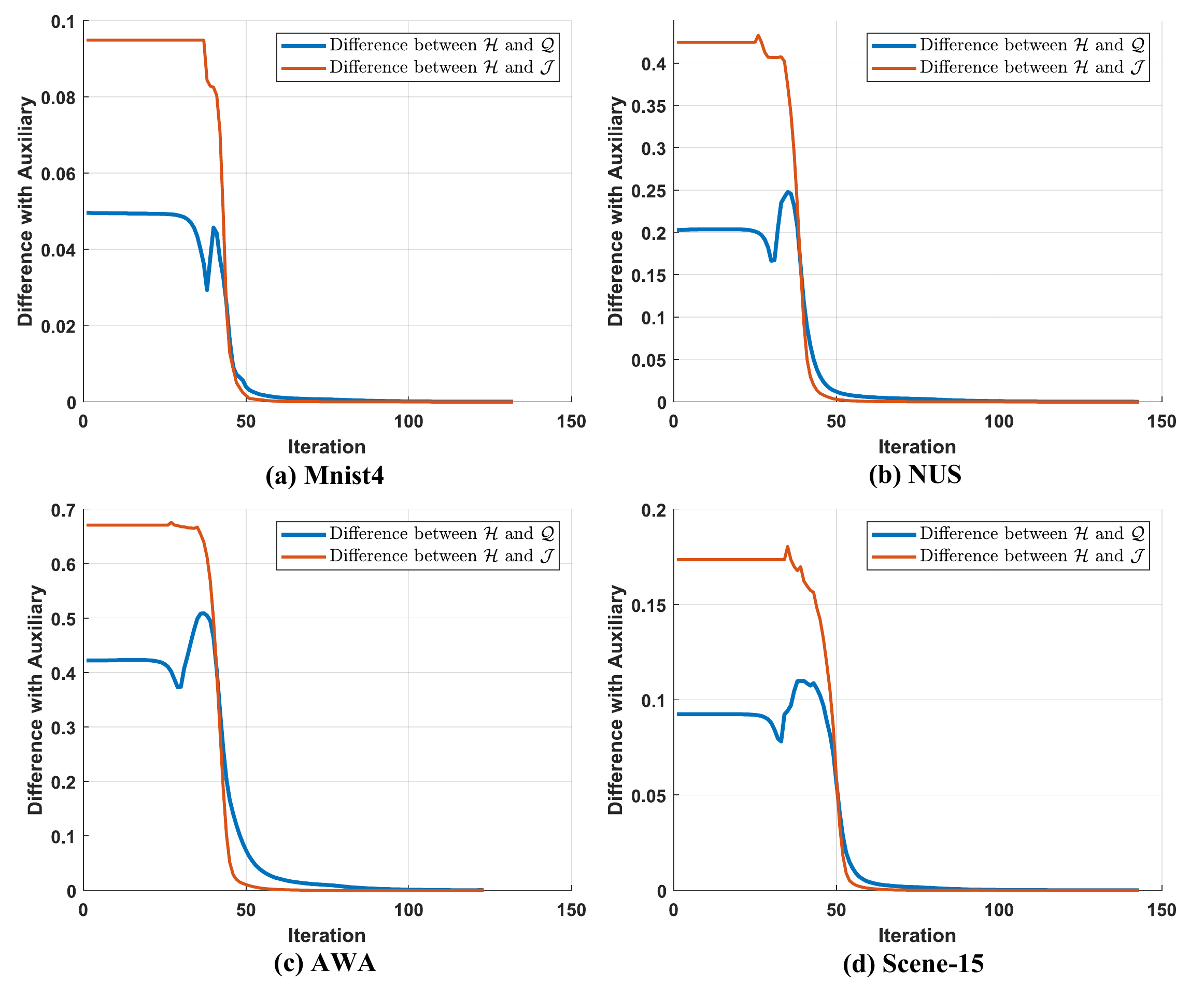}
	\caption{Convergence experiments on Mnist4, NUS, AWA and Scene-15 datasets.}
	\label{fig_sim4}
    \vspace{-5mm}
\end{figure}

\subsubsection{Effect of the number of anchors}
For different datasets, we choose different sets of anchor numbers. For Mnist4, NUS and Scene-15, we choose the number of anchors from $\{k+8,\ k+12,\ k+16,\ k+20,\ k+24,\ k+28,\ 50,\ 100\}$. Since the classification of AWA is 50 and it is much larger than the three datasets mentioned above, we choose the number of anchors from $\{k+8,\ k+10,\ k+12,\ k+14,\ k+16,\ k+18,\ 90,\ 100\}$ for AWA, where $k$ is the number of categories.

As can be seen from Fig \ref{fig_sim}, the clustering results obtained under a different number of anchors are distinctive. When the number of anchors is $k+20,\ k+16,\ k+18,$ and $k+16,$ Mnist4, NUS, AWA, and Scene-15 reach the best results. Meanwhile, as shown in Fig \ref{fig_sim}, we can conclude that it is unnecessary to select more anchors to get better clustering results.

\subsubsection{Effect of the value of $p$}
We set the value of $p$ to be 0.1 to 1.0 with a step of 0.1. We obtained the results of ACC, NMI, and Purity in experiments with different values of p. The clustering results are shown in Fig \ref{fig_sim2}. We can observe that the clustering performance is significantly distinctive under different values of $p$ obviously. The best clustering results are obtained on Mnist4, NUS, AWA, and Scene-15 when the values of $p$ are $0.7,\ 0.2,\ 0.1,$ and $0.1,$ respectively. This may be due to the fact that the tensor Schatten $p$-norm can take full advantage of the complementary information of the multi-view data.

\subsubsection{Effect of the value of $\lambda$}
Since different datasets show different clustering performances under the same values of $\lambda$, we get the values of $\lambda$ from $\{100,\ 1000,\ 10000,\ 100000,\ 1000000\}$ for Mnist4 and Scene-15 while $\{10,\ 100,\ 1000,\ 10000,\ 100000\}$ for NUS and AWA. We conduct experiments under the changing values of $\lambda$ and $p$ to judge the impact of the values of $\lambda$ on ACC as well as the values of $\lambda$ and $p$ on ACC.

The result of the experiment is shown in Fig \ref{fig_sim3}. We can find that different values of $\lambda$ under the same $p$ will cause more distinctive clustering results. On the contrary, ACC does not change significantly with the values of $p$ when $\lambda$ is at some specific value. This indicates that only choose the appropriate $\lambda$ can the tensor Schatten $p$-norm make a difference to the greatest extent.

\subsection{Experiments of convergence}
We optimize the objective function iteratively by introducing two auxiliary variables $\bm{\mathcal Q}$ and $\bm{\mathcal J}$. We test the convergence of our algorithm by checking the difference between $\bm{\mathcal H}$ and the two auxiliary variables.
The result of the experiment is shown in Fig \ref{fig_sim4}. It is evident that when the iteration reaches around 50, the difference decreases significantly until it is about zero finally.

\section{Conclusion}\label{conclusion}
We are concerned in this paper with multi-view clustering based on semi-non-negative tensor factorization (Semi-NTF) with one-side orthogonal constraint. Our proposed model extends NMF to Semi-NTF so that the spatial structure information of the multi-view data can be utilized to improve the clustering performance. In addition, the complementary information embedded in different views is fully leveraged by imposing the tensor Schatten $p$-norm composed of cluster indicator matrices. To diminish the computational complexity, we adopt anchor graphs instead of the original multi-view data.
Also, we provide an optimization algorithm for the proposed method and validate the effectiveness of this approach in extensive experiments on different datasets.

\appendices

\subsection{Proof of the 1st part}
\begin{lemma}[Proposition 6.2 of \cite{lewis2005nonsmooth}]\label{lewis}
Suppose $F: \mathbb{R}^{n_1\times n_2}\rightarrow \mathbb{R}$ is represented as $F(X)=f \circ \sigma(X)$, where $X\in\mathbb{R}^{n_1\times n_2} $ with SVD
 $X=U \mathrm{diag}(\sigma_1, \ldots, \sigma_n) V^{\mathrm{T}}$, $n=\min(n_1, n_2)$, and $f$ is differentiable. The gradient of $F(X)$ at $X$ is
\begin{equation}
\label{deritheorem}
\frac{\partial F(X)}{\partial X}=U \mathrm{diag}(\theta) V^{\mathrm{T}},
\end{equation}
where $\theta=\frac{\partial f(y)}{\partial y}|_{y=\sigma (X)}$.
\end{lemma}
To minimize ${\bm{\mathcal{\bar H}}^{(v)}}$ at step $k+1$ in \eqref{solveH2}, the optimal ${\bm{\mathcal{\bar H}}_{k+1}^{(v)}}$ needs to satisfy the first-order optimal condition
$$2{\bm{\mathcal{\bar H}}_{k+1}^{(v)}}={\bm{\mathcal{\bar Q}}_{k}^{(v)}}+\frac{\bm{\mathcal{\bar{Y}}}_{1,k}}{\mu_k}+{\bm{\mathcal{\bar J}}_{k}^{(v)}}+\frac{\bm{\mathcal{\bar{Y}}}_{2,k}}{\rho_k}.$$ By using the updating rule $\bm{\mathcal{\bar Y}}_{1,k+1} ^{(v)}= {\bm{\mathcal{\bar{Y}}}_{1,k}^{(v)}} + \mu_k({\bm{\mathcal{\bar Q}}_{k}^{(v)}}-{\bm{\mathcal{\bar H}}_{k}^{(v)}}),\ \bm{\mathcal{\bar Y}}_{2,k+1} ^{(v)}= {\bm{\mathcal{\bar{Y}}}_{2,k}^{(v)}} + \rho_k({\bm{\mathcal{\bar J}}_{k}^{(v)}}-{\bm{\mathcal{\bar H}}_{k}^{(v)}})$, we have $$\frac{\bm{\mathcal{\bar Y}}_{1,k+1}^{(v)}}{\mu_k}+\frac{\bm{\mathcal{\bar Y}}_{2,k+1} ^{(v)}}{\rho_k}+2({\bm{\mathcal{\bar H}}_{k}^{(v)}}-{\bm{\mathcal{\bar H}}_{k+1}^{(v)}})=0.$$ According to our assumption $\lim_{k\rightarrow 0}\max\{\mu_k,\rho_k\}({\bm{\mathcal{\bar H}}_{k+1}^{(v)}}-{\bm{\mathcal{\bar H}}_k^{(v)}})=0$, we know $\bm{{\mathcal{Y}}}_{1,k+1},\bm{{\mathcal{Y}}}_{2,k+1}$ is bounded.

To minimize $\mathcal{J}$ at step $k+1$ in \eqref{solve J}, the optimal $\mathcal{J}_{k+1}$ needs to satisfy the first-order optimal condition
$$\lambda\nabla_{\bm{\mathcal {J}}}\|\bm{\mathcal {J}}_{k+1}\|^p_{\Sp}+\rho_k(\bm{\mathcal {J}}_{k+1}-\bm{{\mathcal{H}}}_{k+1}-\dfrac{1}{\rho_k}\bm{{\mathcal{Y}}}_{2,k})=0.$$

Recall that when $0<p<1$, in order to overcome the singularity of $(|\eta|^p)'=p\eta/|\eta|^{2-p}$ near $\eta=0$, we consider for $0<\epsilon\ll 1$ the approximation
$$
	\partial |\eta|^p\approx\dfrac{p\eta}{\max\{\epsilon^{2-p},|\eta|^{2-p}\}}.
$$
Letting $\overline {\bm{\mathcal {J}}}^{(i)}={\overline {\bm{\mathcal {U}}}}^{(i)}\mathrm{diag}\left(\sigma_j(\overline {\bm{\mathcal {J}}}^{(i)})\right){\overline {\bm{\mathcal {V}}}}^{(i)\mathrm{H}},$ then it follows from  Lemma~\ref{lewis} that \begin{align*}\frac{\partial \|{\overline {\bm{\mathcal {J}}}}^{(i)}\|^p_{\Sp}}{\partial{\overline {\bm{\mathcal {J}}}}^{(i)}}={\overline {\bm{\mathcal {U}}}}^{(i)}\mathrm{diag}\left(\dfrac{p\sigma_j(\overline {\bm{\mathcal {J}}}^{(i)})}{\max\{\epsilon^{2-p},|\sigma_j(\overline {\bm{\mathcal {J}}}^{(i)})|^{2-p}\}}\right){\overline {\bm{\mathcal {V}}}}^{(i)\mathrm{H}}.\end{align*}
And  then one can obtain
\begin{align*}
&\dfrac{p\sigma_j(\overline {\bm{\mathcal {J}}}^{(i)})}{\max\{\epsilon^{2-p},|\sigma_j(\overline {\bm{\mathcal {J}}}^{(i)})|^{2-p}\}}\leq \dfrac{p}{\epsilon^{1-p}}\\&\Longrightarrow \left\|\frac{\partial \|{\overline {\bm{\mathcal {J}}}}^{(i)}\|^p_{\Sp}}{\partial{\overline {\bm{\mathcal {J}}}}^{(i)}}\right\|^2_F\leq \sum^{n}_{i=1} \dfrac{p^2}{\epsilon^{2(1-p)}}.
\end{align*}So
$\frac{\partial \|{\overline {\bm{\mathcal {J}}}}\|^p_{\Sp}}{\partial{\overline {\bm{\mathcal {J}}}}}$
is bounded.

Let us denote $\widetilde{\mathbf{F}}_{V} = \frac{1}{\sqrt{V}}\mathbf{F}_{V}$, $\mathbf{F}_{V}$ is the discrete Fourier transform matrix of size $V\times V$,\  $\mathbf{F}^{\mathrm{H}}_{V}$ denotes its conjugate transpose. For $\bm{\mathcal {J}}=\overline {\bm{\mathcal {J}}}\times_3 \widetilde{\mathbf{F}}_{V}$ and using the chain rule in matrix calculus, one can obtain that $$\nabla_{\bm{\mathcal {J}}}\|\bm{\mathcal {J}}\|^p_{\Sp}=\frac{\partial \|{\bm{\mathcal {J}}}\|^p_{\Sp}}{\partial{\overline {\bm{\mathcal {J}}}}}\times_3 \widetilde{\mathbf{F}}_{V}^{\mathrm{H}}$$ is bounded.

And it follows that
\begin{align*}
&\bm{{\mathcal{Y}}}_{1,k+1}=\bm{{\mathcal{Y}}}_{2,k}+\rho_{k}(\bm{{\mathcal{H}}}_{k+1}-\bm{\mathcal {J}}_{k+1})\\&\Longrightarrow \lambda\nabla_{\bm{\mathcal {J}}}\|\bm{\mathcal {J}}_{k+1}\|^p_{\Sp}=\bm{{\mathcal{Y}}}_{2,k+1},
\end{align*}
thus $\bm{{\mathcal{Y}}}_{2,k+1}$ appears to be bounded.

Moreover, by using the updating rule $\bm{\mathcal Y}_1=\bm{\mathcal Y}_1+\mu(\bm{\mathcal H}-\bm{\mathcal Q})$, $\bm{\mathcal Y}_2=\bm{\mathcal Y}_2+\rho(\bm{\mathcal H}-\bm{\mathcal J})$,  we can deduce ($i=1,2$)
\begin{align}
\label{eq:Lk_ieq}
&\bm{\mathcal L} (\bm{\mathcal Q}_{k+1}, \bm{\mathcal G}_{k+1}, \bm{\mathcal H}_{k+1},\bm{\mathcal J}_{k+1};\bm{{\mathcal{Y}}}_{i,k})\\& \leq  \bm{\mathcal L} (\bm{\mathcal Q}_{k}, \bm{\mathcal G}_{k}, \bm{\mathcal H}_{k},\bm{\mathcal J}_{k};\bm{{\mathcal{Y}}}_{i,k}) \nonumber\\&= \bm{\mathcal L} (\bm{\mathcal Q}_{k}, \bm{\mathcal G}_{k}, \bm{\mathcal H}_{k},\bm{\mathcal J}_{k};\bm{{\mathcal{Y}}}_{i,k-1})\nonumber\\
&+\frac{\rho_k+\rho_{k-1}}{2\rho^2_{k-1}}\|\bm{{\mathcal{Y}}}_{2,k}-\bm{{\mathcal{Y}}}_{2,k-1}\|_F^2+ \frac{\|\bm{{\mathcal{Y}}}_{2,k}\|_F^2}{2\rho_k}- \frac{\|\bm{{\mathcal{Y}}}_{2,k-1}\|_F^2}{2\rho_{k-1}}\nonumber\\&+\frac{\mu_k+\mu_{k-1}}{2\mu^2_{k-1}}\|\bm{{\mathcal{Y}}}_{1,k}-\bm{{\mathcal{Y}}}_{1,k-1}\|_F^2+ \frac{\|\bm{{\mathcal{Y}}}_{1,k}\|_F^2}{2\mu_k}- \frac{\|\bm{{\mathcal{Y}}}_{1,k-1}\|_F^2}{2\mu_{k-1}}.\nonumber
\end{align}
Thus, summing two sides of \eqref{eq:Lk_ieq} from $k=1$ to $n$, we have
\begin{equation}
\begin{aligned}
&\bm{\mathcal L} (\bm{\mathcal Q}_{n+1}, \bm{\mathcal G}_{n+1}, \bm{\mathcal H}_{n+1},\bm{\mathcal J}_{n+1};\bm{{\mathcal{Y}}}_{i,n}) \\ \leq & \bm{\mathcal L} (\bm{\mathcal Q}_{1}, \bm{\mathcal G}_{1}, \bm{\mathcal H}_{1},\bm{\mathcal J}_{1};\bm{{\mathcal{Y}}}_{i,0})) \\
+&\frac{\|\bm{{\mathcal{Y}}}_{2,n}\|_F^2}{2\rho_n}- \frac{\|\bm{{\mathcal{Y}}}_{2,0}\|_F^2}{2\rho_{0}}+\sum_{k=1}^n\left(\frac{\rho_k+\rho_{k-1}}{2\rho^2_{k-1}}\|\bm{{\mathcal{Y}}}_{2,k}-\bm{{\mathcal{Y}}}_{2,k-1}\|_F^2\right) \\
+&\frac{\|\bm{{\mathcal{Y}}}_{1,n}\|_F^2}{2\mu_n}- \frac{\|\bm{{\mathcal{Y}}}_{1,0}\|_F^2}{2\mu_{0}}+\sum_{k=1}^n\left(\frac{\mu_k+\mu_{k-1}}{2\mu^2_{k-1}}\|\bm{{\mathcal{Y}}}_{1,k}-\bm{{\mathcal{Y}}}_{1,k-1}\|_F^2\right).
\label{eq:Lk_sum}
\end{aligned}
\end{equation}
Observe that
\[
\sum_{k=1}^{\infty}\frac{\rho_k+\rho_{k-1}}{2\rho_{k-1}^2}<\infty,\sum_{k=1}^{\infty}\frac{\mu_k+\mu_{k-1}}{2\mu_{k-1}^2}<\infty,
\]
we have the right-hand side of \eqref{eq:Lk_sum} is finite and thus $\bm{\mathcal L} (\bm{\mathcal Q}_{n+1}, \bm{\mathcal G}_{n+1}, \bm{\mathcal H}_{n+1},\bm{\mathcal J}_{n+1};\bm{{\mathcal{Y}}}_{i,n})$ is bounded. Notice from \eqref{objective function}
\begin{align}
\label{eq:Ln_bdd}
&\bm{\mathcal L} (\bm{\mathcal Q}_{n+1}, \bm{\mathcal G}_{n+1}, \bm{\mathcal H}_{n+1},\bm{\mathcal J}_{n+1};\bm{{\mathcal{Y}}}_{i,n}) \nonumber\\
&= \sum\limits_{v = 1}^{V} {{\left\| \bm{\mathcal{\bar S}}^{(v)} - \bm{\mathcal{\bar Q}}_{n+1}^{(v)}(\bm{\mathcal{\bar G}}_{n+1}^{(v)})^T \right \|}_F^2} \nonumber\\
&+\lambda\|\bm{\mathcal {J}}_{n+1}\|^p_{\Sp} + \frac{\rho_{n}}{2}\|\bm{\mathcal {H}}_{n+1}-\bm{\mathcal {J}}_{n+1}+\frac{\bm{{\mathcal{Y}}}_{2,n}}{\rho_{n}}\|_F^2\nonumber\\
&+ \frac{\mu_{n}}{2}\sum\limits_{v = 1}^{V}\|\bm{\mathcal{\bar Q}}_{n+1}^{(v)}-\bm{\mathcal{\bar H}}_{n+1}^{(v)}+\frac{\bm{\mathcal{\bar Y}}_{1,n+1} ^{(v)}}{\mu_{n}}\|_F^2,
\end{align}
and each term  of \eqref{eq:Ln_bdd} is nonnegative, following from the boundedness of $\bm{\mathcal L} (\bm{\mathcal Q}_{n+1}, \bm{\mathcal G}_{n+1}, \bm{\mathcal H}_{n+1},\bm{\mathcal J}_{n+1};\bm{{\mathcal{Y}}}_{i,n})$, we can deduce each term of \eqref{eq:Ln_bdd} is bounded. And
$\|\bm{{\mathcal{J}}}_{n+1}\|^p_{\Sp}$ being bounded implies that all singular values of $\bm{{\mathcal{J}}}_{n+1}$ are bounded and hence $\|\bm{{\mathcal{J}}}_{n+1}\|^2_F$ (the sum of squares of singular values) is  bounded. Therefore, the sequence $\{\bm{{\mathcal{J}}}_k\}$ is  bounded.

Because $$\bm{{\mathcal{Y}}}_{1,k+1}=\bm{{\mathcal{Y}}}_{1,k}+\mu_{k}(\bm{{\mathcal{H}}}_{k}-\bm{{\mathcal{Q}}}_{k})\Longrightarrow \bm{{\mathcal{Q}}}_{k}=\bm{{\mathcal{H}}}_{k}+\frac{\bm{{\mathcal{Y}}}_{1,k+1}-\bm{{\mathcal{Y}}}_{1,k}}{\mu_{k}},$$ and in light of the boundedness of $\bm{{\mathcal{H}}}_{k}, \bm{{\mathcal{Y}}}_{1,k}$, it is clear that $\bm{{\mathcal{Q}}}_{k}$ is also bounded.

And from \eqref{solveG1}, it is evident that $\|\bm{\mathcal{\bar G}}_k^{(v)}\|^2_F \leq \|(\bm{\mathcal{\bar S}}^{(v)})^\textrm{T}\|^2_F\|\bm{\mathcal{\bar Q}}_k^{(v)}\|^2_F$, so $\bm{\mathcal{\bar G}}_k^{(v)}$ is also bounded. So $\bm{{\mathcal{G}}}_{k}$ is bounded.

\subsection{Proof of the 2nd part}

From Weierstrass-Bolzano theorem, there exists at least one accumulation point of the sequence $\mathcal{P}_{k}$. We denote one of the points $\mathcal{P}^*=\{\bm{{\mathcal{H}}}^*, \bm{{\mathcal{Q}}}^*, \bm{{\mathcal{G}}}^*, \bm{{\mathcal{J}}}^*, \bm{{\mathcal{Y}}_1}^{*}, \bm{{\mathcal{Y}}}^*_{2}\}$. Without loss of generality, we assume $\{\mathcal{P}_{k}\}^{+\infty}_{k=1}$ converge to $P^*.$

Note that from the updating rule for $\bm{{\mathcal{Y}}}_1$, we have $$\bm{{\mathcal{Y}}}_{2,k+1}=\bm{{\mathcal{Y}}}_{2,k}+\rho_{k}(\bm{{\mathcal{Q}}}_{k}-\bm{{\mathcal{J}}}_{k})\Longrightarrow \bm{{\mathcal{J}}}^*=\bm{{\mathcal{Q}}}^*.$$

Note that from the updating rule for $\bm{{\mathcal{Y}}}_2$, we have $${\bm{\mathcal{\bar{Y}}}_{1,k+1} ^{(v)}}= {\bm{\mathcal{\bar{Y}}}_{1,k}^{(v)}} + \mu_k({\bm{\mathcal{\bar Q}}_{k}^{(v)}}-{\bm{\mathcal{\bar H}}_{k}^{(v)}})\Longrightarrow {\bm{\mathcal{\bar Q}}^{(v)*}}={\bm{\mathcal{\bar H}}^{(v)*}}.$$
In the $\bm{\mathcal{\bar G}}^{(v)}$-subproblem \eqref{solveG1}, we have $$\bm{\mathcal{\bar G}}_k^{(v)} = (\bm{\mathcal{\bar S}}^{(v)})^\textrm{T}\bm{\mathcal{\bar Q}}_k^{(v)}\Longrightarrow{\bm{\mathcal{\bar G}}}^{(v)*} = (\bm{\mathcal{\bar S}}^{(v)})^\textrm{T}\bm{\mathcal{\bar Q}}^{(v)*}.$$

In the $\bm{\mathcal {J}}$-subproblem \eqref{solve J}, we have $$\lambda\nabla_{\bm{\mathcal {J}}}\|\bm{\mathcal {J}}_{k+1}\|^p_{\Sp}=\bm{{\mathcal{Y}}}_{2,k}\Longrightarrow\bm{{\mathcal{Y}}}_1^{*}=\lambda\nabla_{\bm{\mathcal {J}}}\|\bm{\mathcal {J}}^*\|^p_{\Sp}.$$

Therefore, one can see that the sequences $\bm{{\mathcal{H}}}^*, \bm{{\mathcal{Q}}}^*, \bm{{\mathcal{G}}}^*, \bm{{\mathcal{J}}}^*, \bm{{\mathcal{Y}}_1}^{*}, \bm{{\mathcal{Y}}}^*_{2}$ satisfy the KKT conditions of the Lagrange function \eqref{objective function}.

\subsection{Anchor Selection And Graph Construction}\label{anchorSelection}
Inspired by~\cite{li2020multiview}, we adopt directly alternate sampling (DAS) to select anchors.

First of all, with the given data matrices $\{{\bf{X}}^{(v)}\}_{v=1}^V$, we concatenate the data matrix of each view along the feature dimension. The connected feature matrix $\mathbf X \in{\mathbb{R}}^{{n} \times {d}}$ can be represented as $\mathbf X = [\mathbf X^{(1)};\mathbf X^{(2)};\cdots;\mathbf X^{(v)}]$, where $d$ is the sum of the number of features in each view.
Let $\theta_i$ represent the $i$-th sample of the $d$-dimensional features, which can be calculated as
\begin{equation}
\theta_i = \sum_{j=1}^{dT} Tra(X_{ij}),
\end{equation}
where $dT=\sum_{v=1}^{V} d_v$, and $Tra(\cdot)$ represents the transformation of the raw features. Specifically, if the features are negative, we process the features of each dimension by subtracting the minimum value in each dimension. Then we obtain the score vector $\bm{\theta} = [\theta_1,\theta_2,\cdots, \theta_n] \in {\mathbb{R}}^{n}$. We choose the point where the maximum score is located as the anchor. The position of the largest score is
\begin{equation}\label{selectAnchor1}
	Index = \arg \max_i \theta_i.
\end{equation}
Then the 1st anchor of the $v$-th view is $b_1^{(v)} = x_{Index}^{(v)}$.

After that, let $\theta_{Index}$ be the score of the anchor selected from the last round, then we normalize the score of each sample by:
\begin{equation}\label{selectAnchor2}
	\theta_i \leftarrow \frac{\theta_i}{\max \bm{\theta}}, (i=1,2,\cdots,n)
\end{equation}
Then the score $\theta_i$ can be updated as
\begin{equation}\label{selectAnchor3}
	\theta_i \leftarrow \theta_i \times (1-\theta_i).
\end{equation}

Finally, we repeat  (\ref{selectAnchor1}) -  (\ref{selectAnchor3}) $m$ times to select $m$ anchors. After selecting $m$ anchors, we construct an anchor graph of each view $\mathbf S^{(v)}$, in the same way, as~\cite{li2020multiview}.

\ifCLASSOPTIONcompsoc
\else
\fi
\ifCLASSOPTIONcaptionsoff
  \newpage
\fi

{\small
\bibliographystyle{IEEEtran}
\bibliography{egbib}
}

\end{document}